\documentclass{article}

\usepackage[preprint]{neurips_2022} 
\usepackage[ruled, lined, linesnumbered, commentsnumbered, longend]{algorithm2e}

\usepackage{amsmath, amsthm}

\usepackage[utf8]{inputenc} 
\usepackage[T1]{fontenc}    
\usepackage[bookmarks=false]{hyperref}       
\usepackage{url}            
\usepackage{booktabs}       
\usepackage{amsfonts}       
\usepackage{nicefrac}       
\usepackage{microtype}      
\usepackage{xcolor}         
\usepackage{graphicx}
\usepackage{bbm}
\usepackage{subcaption, floatrow}

\def\eqref#1{equation~\ref{#1}}

\def\ceil#1{\lceil #1 \rceil}

\def\1{\bm{1}}

\def\rvs{{\mathbf{s}}}

\def\rvx{{\mathbf{x}}}
\def\rvy{{\mathbf{y}}}

\def\gA{{\mathcal{A}}}

\def\gG{{\mathcal{G}}}

\def\gP{{\mathcal{P}}}

\def\gS{{\mathcal{S}}}

\def\sN{{\mathbb{N}}}

\def\Prob{{\mathbb{P}}}

\newcommand{\E}{\mathbb{E}}

\newcommand{\R}{\mathbb{R}}

\newtheorem{theorem}{Theorem}
\newtheorem{assumption}{Assumption}
\newtheorem{lemma}{Lemma}
\newtheorem*{prop*}{Proposition}
\newtheorem{definition}{Definition}

\newcommand{\D}{\mathcal{D}}
\newcommand{\Dcal}{\mathcal{D}_{cal}}
\newcommand{\Dtrain}{\mathcal{D}_{train}}

\title{Conformal Methods for Quantifying Uncertainty \\ in Spatiotemporal Data: A Survey}

\author{%
  Sophia Sun\\
 University of California, San Diego\\
  \texttt{sophiasun@eng.ucsd.edu}}

\begin{document}

\maketitle

\begin{abstract}
Machine learning methods are increasingly widely used in high-risk settings such as healthcare, transportation, and finance. In these settings, it is important that a model produces calibrated uncertainty to reflect its own confidence and avoid failures. In this paper we survey recent works on uncertainty quantification (UQ) for deep learning, in particular distribution-free Conformal Prediction method for its mathematical properties and wide applicability. We will cover the theoretical guarantees of conformal methods, introduce techniques that improve calibration and efficiency for UQ in the context of spatiotemporal data, and discuss the role of UQ in the context of safe decision making.

\end{abstract}

\section{Introduction}

Let $\mathcal{D} = (z_1, \ldots, z_n)$ be a dataset of size $n$. We denote $z_i = (x_i,y_i)$ as a sample of an input and output pair that follow the distribution $\mathcal{P}$, where $i$ is the data index. Let the input space $\mathbf{X}$ and target space $ \mathbf{Y}$ be two measurable spaces, their Cartesian product $\mathbf{Z} = \mathbf{X} \times \mathbf{Y}$ is the sample space. Consider the following problem: given a desired coverage rate $1-\alpha \in (0,1)$, we want to construct a prediction region $\Gamma^{1-\alpha}: \mathbf{X} \rightarrow \{\text{subsets of } \mathbf{Y}\}$ such that for a new data pair $(X,Y) \sim  \mathcal{P}$, we have

\begin{equation}
    \Prob_{(X,Y) \sim \gP} (Y \in \Gamma^{1-\alpha}(X)) \geq 1-\alpha
\label{eq:validity}
\end{equation}

In the case where $Y \in \Gamma^{1-\alpha}(X)$, the prediction region \textit{covers} $Y$. We say the prediction region is \textit{valid} if it satisfies Equation \ref{eq:validity}. Hence, Equation \ref{eq:validity} is also known as the \textit{validity condition} or \textit{coverage guarantee}. An example of a prediction region that we are familiar with in real life is hurricane forecasts (figure \ref{fig:cone}). In the spatiotemporal setting, uncertainty quantification can often be visualized as a "cone of uncertainty".

\begin{figure}
    \centering
    \includegraphics[width=0.7\linewidth]{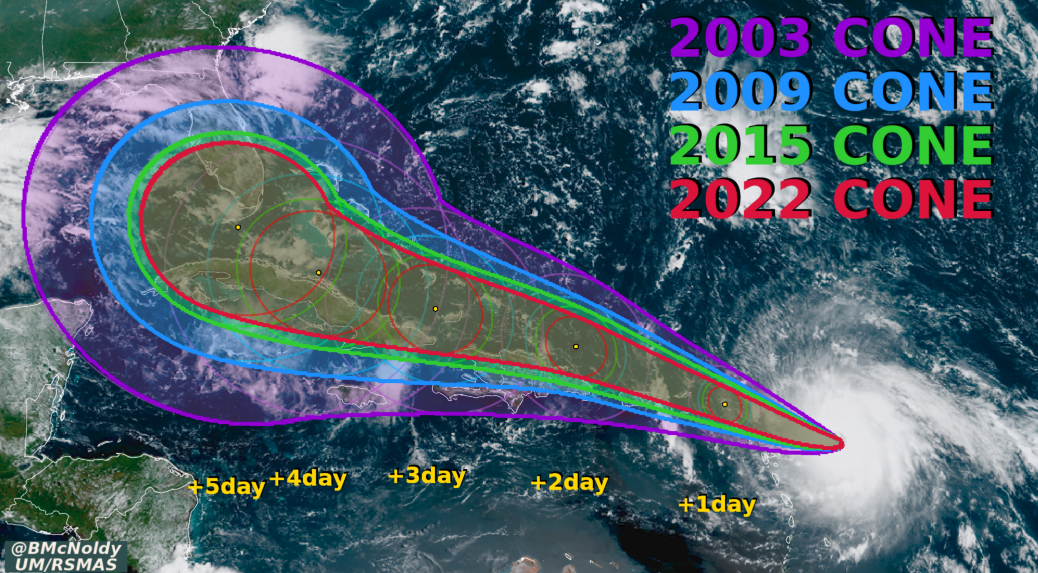}
    \caption{The "cone of uncertainty" for hurricane forecasts.  \protect\cite{molina2021striving}}
    \label{fig:cone}
\end{figure}

Quantifying the uncertainty of a model's predicted outcome is critical for risk assessment and decision making. Prediction regions allow us to bound random variables of interest and know that, with a specified probability, where the observation could fall. Consider the case of pandemic forecasts, for example: compared to a predicted number of future cases without a confidence metric, a prediction interval with $99\%$ confidence is much more useful for local health facilities, as they can safely plan for the worst-case scenario. 

The problem of building a prediction region can be approached in many ways. One may model the distribution of data via maximum likelihood or Bayesian methods, or use a \textit{distribution-free} method to capture the uncertainty - a method is \textit{distribution-free} when it creates valid prediction sets without making distributional assumptions on the data. Among them, Conformal Prediction (CP) is a straightforward algorithm to generate prediction sets for \textit{any} underlying prediction model. 

This survey will be organized as follows. We will provide a brief survey for uncertainty quantification methods in Section 2. Conformal prediction \cite{vovk2005algorithmic} allows us to satisfy the validity condition (Equation \ref{eq:validity}) with no assumptions on the common distribution $\mathcal{P}$, hence the name \textit{distribution-free} uncertainty quantification. Section 3 will introduce the algorithm and theory of conformal prediction, recent advancements for improved calibration and efficiency, and methods for applying conformal prediction to spatiotemporal data. In 4 we present methods that utilize conformal prediction for safe decision making. We will also discuss the limits of distribution-free UQ methods in section 4.2. We conclude in 5, and discuss several future directions for research on conformal algorithms and safe decision-making.

\section{Brief Overview of UQ methods}
\paragraph{Bayesian Uncertainty Quantification}
Bayesian approaches represent uncertainty by estimating a distribution over the model parameters given data, and then marginalizing these parameters to form a predictive distribution.
Bayesian neural networks (BNN) \cite{mirikitani2009recursive, mackay1992bayesian} were proposed to learn through Bayesian inference with neural networks. As modern neural networks often contain millions of parameters, the posterior over these parameters becomes highly non-convex, rendering inference intractable.
Modern approaches performs approximate Bayesian inference by Markov chain Monte Carlo (MCMC) sampling \cite{welling2011bayesian, neal2012bayesian, chen2014stochastic} or variational inference (VI) \cite{graves2011practical, kingma2013auto, kingma2015variational, blundell2015weight, louizos2017multiplicative}. 
In practice, stochastic gradient MCMC methods are prone to approximation errors, and can be to difficult to tune \cite{mandt2017stochastic}. Variational inference methods have seen strong performance on moderately sized networks, but are empirically found to be difficult to train on larger architectures \cite{he2016deep,blier2018description}.
MC-dropout methods
\cite{gal2016theoretically, gal2017concrete} view dropout at test time to be approximate variational Bayesian inference. MC dropout methods have found success in efficiently approximate the posterior \cite{kendall2017uncertainties}, but calibration of the resulted distribution has been found to be difficult \cite{alaa2020frequentist}. We refer the readers to recent surveys like \cite{abdar2021review} for the rich body of work in Bayesian uncertainty quantification.

\paragraph{Mixed / Re-calibration Methods}
In practice, Bayesian uncertainty estimates often
fail to capture the true data distribution due to intractability of Bayesian optimization \cite{lakshminarayanan2017simple, zadrozny2001obtaining}. There exists a line of research exploring re-calibration of neural network models - \cite{lakshminarayanan2017simple} proposed using ensembles of several networks for enhanced calibration, and incorporated an adversarial loss function to be used when possible as well. \cite{guo2017calibration} proposed temperature scaling, a procedure which uses a validation set to rescale the logits of deep neural network outputs for enhanced calibration. \cite{kuleshov2018accurate} and \cite{kull2019beyond} propose calibrated regression using similar rescaling techniques, the latter through stochastic weight averaging.  \cite{minderer2021revisiting}

\paragraph{Frequentist and Distribution-free Uncertainty Quantification}
Frequentist UQ methods emphasize the robustness against variations in the data. These approach either rely on resampling the data or by learning an interval bound to encompass the dataset. Among them are ensemble methods such as bootstrap \cite{efron2016, alaa2020frequentist}, model ensembles \cite{lakshminarayanan2017simple, pearce2018high,jackknife+after}; interval prediction methods include quantile regression \cite{tagasovska2019single, gasthaus2019probabilistic,takeuchi2006nonparametric}, interval regression through proper scoring rules \cite{kivaranovic2020adaptive, wu2021quantifying}; rank tests \cite{rankedtest1, rankedtest2} and permutation tests \cite{permtest}. Many of these frequentist methods satisfy the validity condition asymptotically and can be categorized as distribution-free UQ techniques as they are (1) agnostic to the model and (2) agnostic to the data distribution.

\paragraph{Conformal Prediction} 
Conformal prediction is an important example of distribution-free UQ method - a key difference between conformal methods from the ones introduced in the previous section is that it achieves the validity condition in finite samples. Originally working on studies of finite random sequences, Vladimir Vovk and Glenn Shafer coined the name and developed the core theory for online and split conformal prediction in their 2005 book \textit{Algorithmic Learning in a Random World} \cite{vovk2005algorithmic}. Conformal prediction has since become popular because of its simplicity and low computational cost properties \cite{shafer2008tutorial, angelopoulos2021gentle} and has shown promise in many realistic applications \cite{eklund2015application, angelopoulos2022image, lei2021conformal}. Current research on conformal prediction can be roughly split into three branches, shown in figure \ref{fig:venn}. The first is to develop conformal prediction algorithms for different machine learning algorithms, such as quantile regression \cite{romano2019conformalized, sesia2020comparison}, k-Nearest Neighbors \cite{papadopoulos2011regression}, density estimators \cite{izbicki2019flexible}, survival analysis \cite{teng2021t,candes2021conformalized}, risk control  \cite{bates2021distribution}, et cetera. The second is works on relaxing the exchangability assumption about data distribution so we can make guarantees in the context of, for example, distribution shifts \cite{tibshirani2019conformal, podkopaev2021distribution, barber2022conformal, hu2020distribution, gibbs2021adaptive}. The third line is improving efficiency (i.e. how sharp the prediction regions are while  validity holds) of conformal prediction \cite{messoudi2020multitarget, messoudi2021copula, romano2020classification, yang2021finite, sesia2020comparison}. The algorithms for time series that we will focus on in this survey can generally be categorized into the latter two branches.

\begin{figure}
    \centering
    \includegraphics[width=0.99\linewidth]{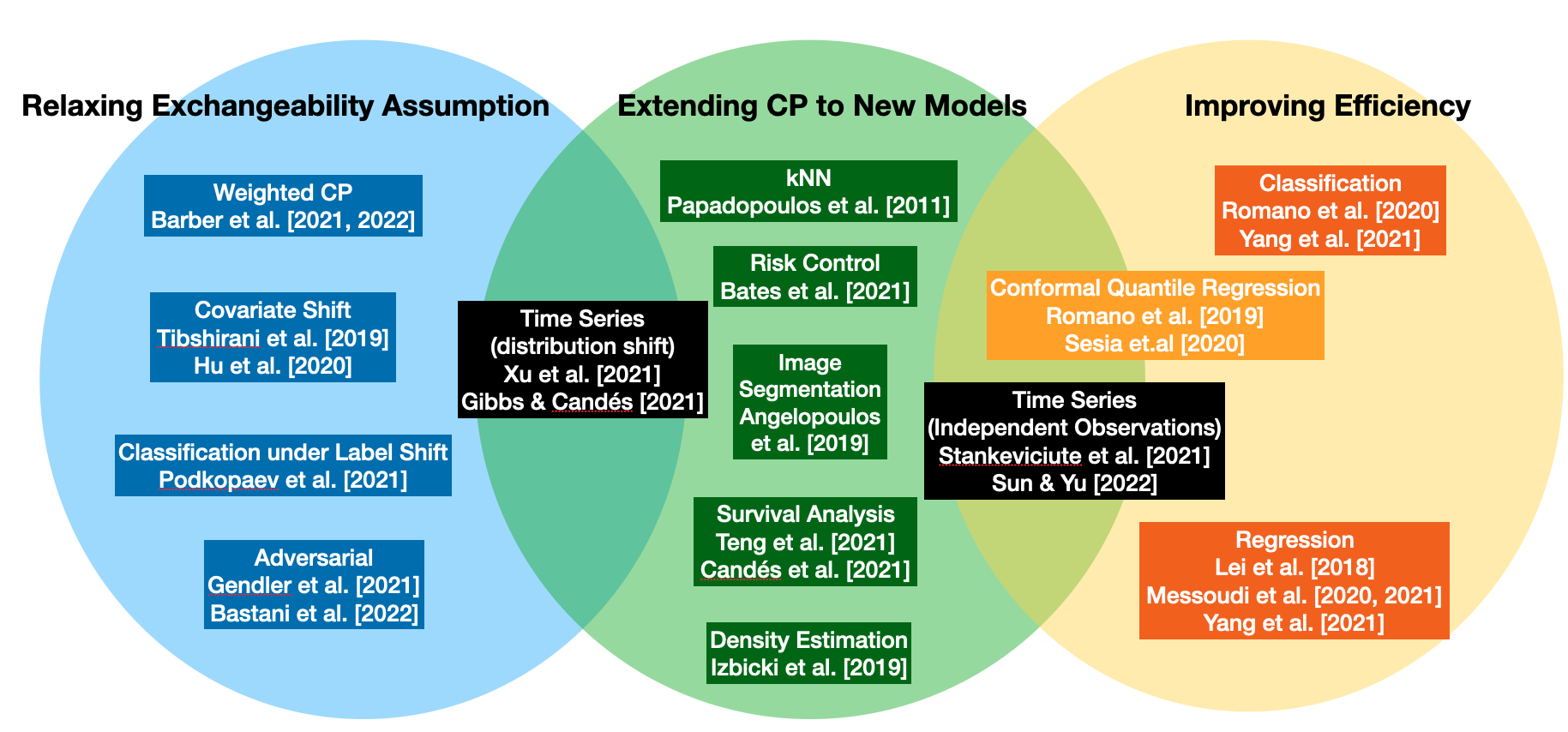}
    \caption{Venn Diagram of recent research on Conformal prediction.}
    \label{fig:venn}
\end{figure}

\begin{table*}
\centering
\begin{tabular}{ c | c cc c  } 
 \toprule
 Algorithm & Data & Distribution Assumption & Coverage Guarantee \\
 \midrule
(split) Conformal \cite{vovk2005algorithmic} & Regression & Exchangeable & $1-\alpha$ \\
EnbPI \cite{xu2021conformal} & single time series & strongly mixing errors & $\approx 1-\alpha$   \\
ACI \cite{gibbs2021adaptive} & single time series & None (online) & Asymptotically $\approx 1-\alpha$  \\
CF-RNN \cite{Schaar2021conformaltime} & independent time series & exchangeable time series & $1-\alpha$   \\ 
Copula-RNN (ours) &  independent time series & exchangeable time series &  $1-\alpha$ \\
\bottomrule
\end{tabular}
\caption{A summary of Algorithms in Sections 3.2 - 3.4. } 
\label{tab:guarantee}
\end{table*}

\section{Conformal Prediction Algorithms}

\subsection{Conformal Prediction}
This section introduces the necessary notation and definitions, and formally presents the idea of conformal prediction. 

Recall that the goal of conformal prediction is to constructing a prediction region $\Gamma^{1-\alpha}: \mathbf{X} \longrightarrow \{\text{subsets of } \mathbf{Y}\}$ given a new data pair $(X,Y) \sim  \mathcal{P}$ and a desired coverage rate $1-\alpha \in (0,1)$ such that $\Gamma^{1-\alpha}(X)$ is valid, i.e. 
\begin{equation}
    \Prob (Y \in \Gamma^{1-\alpha}(X)) \geq 1-\alpha
\label{eq:valid2}
\end{equation}

\paragraph{Validity and Efficiency} There is a trade-off between validity and efficiency, as one can always set $\Gamma^{1-\alpha}$ to be infinitely large to satisfy Equation \ref{eq:valid2}. In practice, we want to achieve that the measure of the confidence region (e.g. its area or length) to be as small as possible, given that the validity condition holds. This is known as \textit{efficiency} or \textit{sharpness} of the confidence region.

\paragraph{Training and calibration datasets} Let $\mathcal{D} = (z_1, \ldots, z_n)$ be a dataset of size $n$. The inductive conformal prediction method begins by splitting $\mathcal{D}$ into two disjoint subsets, the proper training set $\mathcal{D}_{train} = (z_1, \ldots, z_m)$ of size $m < n$ and the calibration set $\mathcal{D}_{cal}$ of size $n-m$. 

\paragraph{Nonconformity score}
An important component of conformal prediction is the nonconformity score $ \gS : \mathbf{Z}^m \times \mathbf{Z} \rightarrow \mathbb{R}$, a function to capture how well a sample $z$ \textit{conforms} to the proper training set. Informally, a high value of $\gS((x,y), \Dcal)$ indicates that the point $(x,y)$ is atypical relative to the points in $\Dcal$.
A common choice of the non-conformity score $\gS$ in a regression settings is the L2-loss on the test sample $(x,y)$:
\begin{equation}
   \gS((x,y), \Dcal) := |y - \hat{f}(x)|
\end{equation}

where $\hat{f}: \mathbf{X}\longrightarrow \mathbf{Y}$ is a prediction model trained with the proper training set $\mathcal{D}_{train}$. In the following sections, we refer to the non-conformity score of a data point $z$ given a prediction model $\hat{f}$ as $\gS(z, \hat{f})$ for simplicity.

\paragraph{Algorithm} 
The conformal prediction algorithm consists of two parts, the training and the calibration process. During the training process, we fit a machine learning model indicated by $\hat{f}$ to the training data set $\Dtrain$. The validity guarantee of the conformal algorithm holds regardless of the underlying model $\hat{f}$. In the calibration process, we calculate the nonconformity score $s_i = \gS(z_i, \hat{f})$ for each $z_i \in \Dcal$. We can then create the confidence region based on the nonconformity scores. We outline the full algorithm in Algorithm \ref{alg:cp}.

\begin{algorithm}[t!]
    \SetKwInput{Input}{Input}
    \SetKwInput{Output}{Output}
    \Input{Dataset $\D = {(x_i, y_i)}_{i = 1, \ldots, n}$, regression algorithm $\gA$, test point $x_{n+1}$, target confidence level $1-\alpha$}
    \Output{Prediction region $\Gamma^{1-\alpha}(x_{n+1})$.}
    \hrulefill\\
    Randomly split dataset $\D$ into training and calibration datasets $\D = \Dtrain \cup \Dcal$ where $|\Dtrain| = m$ and $|\Dcal| = n-m$;\\
    Train machine learning model $\hat{f} = \gA(\Dtrain)$;\\
    Initialize nonconformity scores $\mathbf{s}_{cal} = \{\}$;\\ 
    \For{$(x_i, y_i) \in \mathcal{D}_{cal}$}{
        $\mathbf{s}_{cal} \leftarrow \mathbf{s} \cup \{ \gS((x_i, y_i), \hat{f}) \}$;
    }
   Calculate the $(1-\alpha)$-th quantile  $q_{1-\alpha}$ of $\mathbf{s}_{cal} \cup \{\infty \}$; \\
    \KwRet{$\Gamma^{1-\alpha}(x_{n+1}) = \{y: \gS( (x_{n+1},y), \hat{f}) \leq q_{1-\alpha} \}$}
 \caption{Conformal Prediction}
 \label{alg:cp}
\end{algorithm}

\paragraph{Theoretical Guarantees} We will prove that the prediction regions generated by Algorithm \ref{alg:cp} is \textit{valid} (Equation \ref{eq:validity} )if the dataset $\D \sim \gP$ is \textit{exchangeable} (assumption \ref{as:exchange}).

\begin{assumption} [exchangeability]
Assume that the data points ${z_i} \in \mathcal{D}$ and the test point $z_{n+1}$ are exchangeable. Formally, $z_i$ are exchangeable if arbitrary permutation leads to the same distribution, that is,
\[
(z_1, \ldots, z_{n+1} ) \stackrel{d}{=}  (z_{\pi(1)}, \ldots, z_{\pi(n+1)} )
\]
with arbitrary permutation $\pi$ over $\{1, \ldots, n+1\}$
\label{as:exchange}
\end{assumption}

Define the empirical $p$-quantile of a set of nonconformity scores $s_{1:n} = \{s_1, \ldots, s_n\}$ as:
\begin{equation}
    Q(p, s_{1:n}) := \inf\{s': (\frac{1}{n} \sum_{i=1}^n \mathbbm{1}_{s_i \leq s'}) \geq p \}
    \label{eq:quantile}
\end{equation}

\begin{lemma} [Lemma 1 in \cite{tibshirani2019conformal}]
If $s_1 \ldots s_{n+1}$ are exchangeable random variables, then for any $p\in(0,1)$, we have
\[
\Prob(s_{n+1} \leq Q(p, s_{1:n}\cup \{\infty\}) ) \geq p
\]
Furthermore, if ties within $s_{1:n}$ occur with probability zero, then the above probability is upper bounded by $p+1/(n+1)$.
\label{lm:quantile}
\end{lemma}

\begin{proof}
 Let $q = Q(p, s_{1:n}\cup \{\infty\})$. Note that if $s_{n+1} > q$, then $Q(p, s_{1:n}\cup \{\infty\}) = Q(p, s_{1:n}\cup \{s_{n+1}\})$, the quantile remains unchanged. This fact can be written as:

\[
s_{n+1} > Q(p, s_{1:n}\cup \{\infty\}) \Longleftrightarrow s_{n+1} > Q(p, s_{1:n+1}) 
\]

and equivalently 

\[
s_{n+1} \leq Q(p, s_{1:n}\cup \{\infty\}) \Longleftrightarrow s_{n+1} \leq Q(p, s_{1:n+1}) 
\]

The above describes the condition where $s_{n+1}$ is among the $\ceil{p(n+1)}$ smallest of $s_{1:n}$. By exchangability, the probability of $s_{n+1}$'s rank among $s_{1:n}$ is uniform. Therefore,

\[
\Prob (s_{n+1} \leq Q(p, s_{1:n+1})) = \Prob(s_{n+1} \leq Q(p, s_{1:n}\cup \{\infty\})  =  \frac{\ceil{p(n+1)}}{(n+1)} \geq p
\]

Which proves the lower bound. When there are no ties, we have $\frac{\ceil{p(n+1)}}{(n+1)} \leq p + \frac{1}{(n+1)}$ which proves the upper bound.
\end{proof}

\begin{theorem}[Validity of the conformal prediction algorithm\cite{vovk2005algorithmic,lei2018distribution}]

Let $\D \sim \gP$ be an exchangeable dataset. Then for any score function $\gS$ and any significance level $\alpha \in (0,1)$, the prediction set given by algorithm \ref{alg:cp} 
\[
\Gamma^{1-\alpha}(x_{n+1}) = \{y: \gS( (x_{n+1},y), \hat{f}) \leq Q(1-\alpha, \mathbf{s}_{cal}) \}
\]

satisfies the validity condition, i.e.
\[
\Prob(y_{x+1} \in \Gamma^{1-\alpha}(x_{n+1}) ) \geq 1-\alpha
\]
\label{the:validity}
\end{theorem}

\begin{proof}
Note that if $\D$ is exchangeable, so is $\Dcal$, and hence so is the set of nonconformity scores $\mathbf{s}_{cal}$. Therefore, by lemma \ref{lm:quantile} we have:
\[
\Prob(y_{x+1} \in \Gamma^{1-\alpha}(x_{n+1}) ) = \Prob \left( \gS( (x_{n+1},y), \hat{f}) \leq Q(1-\alpha, \mathbf{s}_{cal} \cup \{\infty\}) \right) \geq 1-\alpha
\]
\end{proof}

It is easy to prove that the upper bound in lemma \ref{lm:quantile} also applies for the prediction set in theorem \ref{the:validity}. More precisely, if ties between the nonconformity scores $\mathbf{s}_{cal}$ occur with probability zero, then 
\[
1-\alpha \leq \Prob(y_{x+1} \in \Gamma^{1-\alpha}(x_{n+1}) ) \leq 1-\alpha + 1/(n+1)
\]

\paragraph{Remark on exchangability} Exchangeability is a weaker assumption on data than the i.i.d. assumption, hence it is reasonable for many application scenarios. In later sections of this survey we will explore conformal algorithms that relaxes the exchangability assumption and adaptive to data distribution shifts in an online fashion.

\paragraph{Split vs. Full conformal regression} The procedures described above and in Algorithm \ref{alg:cp} is generally known as \textit{inductive} conformal prediction or \textit{split} conformal prediction. It is developed on the basis of the \textit{full} conformal prediction, where we don't split the training set into proper training and calibration sets, but performs $n$ folds of hold-one-out training, and use the held-out sample for calibration. The full conformal prediction method requires training the model $n$ times, and hence is unsuitable for settings where model training is computationally expensive, such as deep learning. The scope of this survey is limited to the regression setting, but conformal prediction can be used for many  machine learning tasks, including and not limited to classification, image segmentation and outlier detection. We refer readers to  \cite{vovk2005algorithmic, angelopoulos2021gentle} for a more thorough introduction. 

\begin{figure}[!ht]
    \centering
    \includegraphics[width=0.4\linewidth]{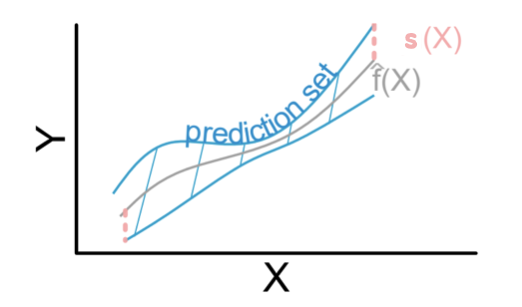}
    \caption{A visualization of the conformal prediction interval in a 1-D regression setting from \protect\cite{angelopoulos2021gentle} }
    \label{fig:interval}
\end{figure}

\paragraph{Adaptive prediction sets} One may notice that if we choose nonconformity score $\gS((x,y),\hat{f}) = |y-\hat{f}(y)|$ then we will end up with the same size of prediction set for any input $x$ following the conformal prediction algorithm \ref{alg:cp}. From an UQ perspective, this is undesirable - we want the procedure to return larger sets for harder inputs and smaller sets for easier inputs. Adaptability is an active line of research within the conformal prediction community, with approaches including, but not limited to:

\begin{itemize}
    \item Normalized nonconformity scores \cite{lei2018distribution}. The most common method for adaptation is to scale an non-adaptive score function with a heuristic function $u(x)$ that capture the uncertainty of the model. That is, $u(x)$ is large when model is uncertain and small otherwise. Define the score function as 
    \[
    \gS((x,y),\hat{f}) = \frac{|y-\hat{f}(x)|}{u(x)}
    \]
     some common choices of $u(x)$ \cite{angelopoulos2021gentle, svensson2018conformal} are 
     \begin{itemize}
         \item Standard deviation to the center of training data mean. $u(x):= \hat{\sigma}(x)$
         \item Distance to training center (in Euclidean space or learned latent space such as PCA)
         \item Nearest neighbour. Uncertainty was derived based on the average distance to the $k-$nearest neighbors learned latent space such as PCA or t-SNE.
         \item Error prediction model model \cite{norinder2014introducing,cortes2016improved} where a machine learning model was trained to predict the expected error given inputs. 
     \end{itemize}
     However, the estimated heuristic $u(x)$ can be (1) unstable and (2) lack direct relation to the quantiles $\alpha$ of the target space distribution. 
    \item Feature-space conformal regression  \cite{bates2021distribution,angelopoulos2021gentle}, where the conformal prediction algorithm is performed on the feature space and then projected to the output space. In the deep learning setting, the feature space can be latent spaces before the final layer(s) for the neural network. \cite{featurecp} proved that validity in the feature space implies validity in the output space, given some smoothness assumption on the prediction layer. 
    \item Conformalized quantile regression (CQR) \cite{steinwart2011estimating, romano2019conformalized}. CRQ allows for varying lengths over due to the learned quantile functions, and allows for asymmetric prediction intervals and a because of the separately fitted upper and lower bounds. A quantile regression algorithm \cite{takeuchi2006nonparametric} attempts to learn the $\gamma$ quantile of $Y_{test}|X_{test} = x$ for each possible value of $x$. We denote the fitted model as $\hat{t}_\gamma(x)$ - by definition, the probability of  $Y_{test}|X_{test} = x$ lies below $\hat{t}_{0.05}(x)$ is $5\%$ and below $\hat{t}_{0.95}(x)$ is $95\%$, hence the interval $[\hat{t}_{0.05}(x), \hat{t}_{0.95}(x)]$ to have $90\%$ coverage. With this intuition, we define the nonconformity score to be the distance from $y$ to the predicted quantile interval, given a desired coverage rate of $1-\alpha$:
    \[
    \gS(x,y) = \max\{ \hat{t}_{\alpha/2}(x)-y, y - \hat{t}_{1-\alpha/2}(x)\} 
    \]
   Figure \ref{fig:quantile} illustrates the result of different adaptive techniques with a synthetic dataset. 
\end{itemize}

\begin{figure}
        \centering
        \includegraphics[width=0.95\linewidth]{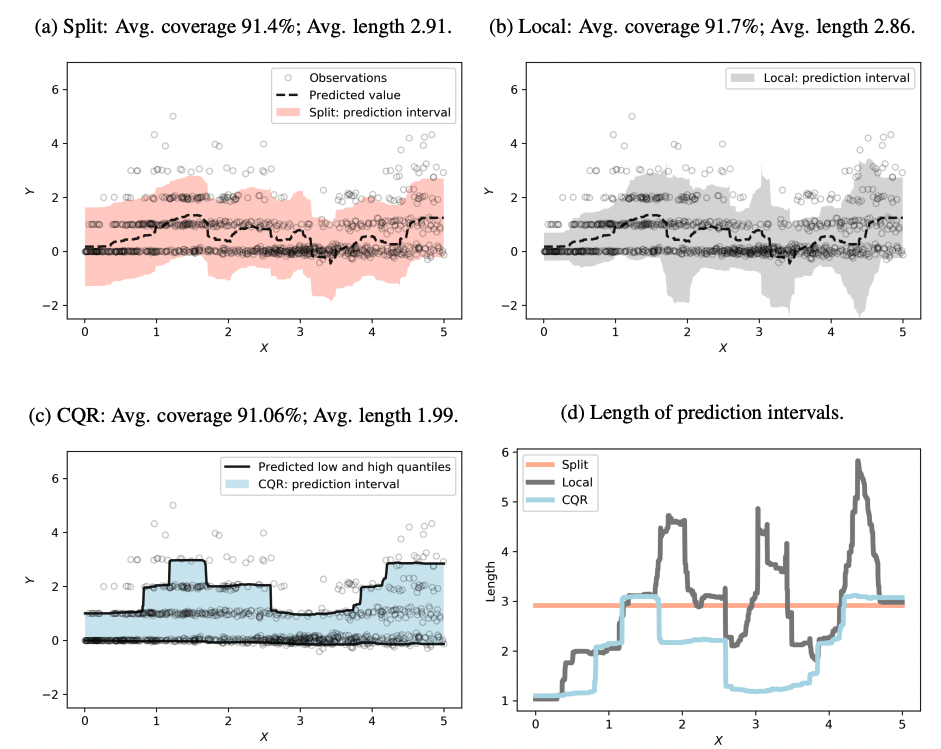}
        \caption{Prediction interval on simulated data, where noise is heteroscedastic and with outliers from \protect\cite{romano2019conformalized}. (a) standard split conformal prediction (algorithm \ref{alg:cp}), (b) normalized conformity score with $u(x) = \hat{\sigma}(x)$ is the mean absolute deviation (MAD) of $|Y - \hat{f}(x)|$ given $X = x$, and (c) Conformalized Quantile Regression (CQR). Observe that while all three methods achieve the $\geq 90\%$ coverage guarantee, the intervals generated by CQR are much more efficient.}
        \label{fig:quantile}
    \end{figure}

The adaptability techniques introduced above are developed for the exchangeable setting. While some methods are modular to the specific conformal algorithm (such as choosing a scaled nonconformity score), further research is needed for adaptability in the time series setting. Formalization of the adaptive property and its implications are further discussed in section 4.2.


\subsection{Conformal prediction for time series under distribution shift}

In this section, we will look at recent developments of extending conformal prediction's validity guarantees to the time series setting. This is a nontrivial task, as real-world time series data often are non-exchangeable. Distribution shifts can happen, for example, in financial markets where market behavior can dramatically change because of a new legislation or policy, and in epidemiological data when there are seasonal and infrastructural shifts. 


\subsubsection{Sequential Distribution-free Ensemble Batch Prediction Intervals (EnbPI \cite{xu2021conformal})}  
The problem setting is we have a unknown model $f: \R^d \rightarrow \R$ where $d$ is the dimension of the feature vector, the data pairs $(x,y)$ are generated in a time-series fashion

\[
Y_t = f(X_t) + \epsilon_t, \quad t= 1,2,\ldots
\]

where the noise term $\epsilon_t$ are identically distributed to a common cumulative distribution function (CDF) $F$, but doesn't have to be independent.. Xu et al. tackles the nonexchangeable problem by (1) training an ensemble of leave-one-out predictors, and (2) deploying a sliding window for calibration and re-calibrates as the time series continues. The algorithm is presented in Algorithm \ref{alg:enbpi}. The ensemble training (line 3-13 in Algorithm \ref{alg:enbpi}) is similar to a bootstrap method \cite{jackknife+after}. Line 19-24 in Algorithm \ref{alg:enbpi} shows the procedure of leveraging recent data for calibration. It allows the prediction intervals to adapt to distributional shifts without having to retrain $\hat{f}$.  The authors show that, under the assumption of the error terms $\{\epsilon_t\}_{t\geq 1})$ are strongly mixing and stationary, the proposed algorithm achieves a coverage approximate validity guarantee. (see Theorem 1 of \cite{xu2021conformal}). 

\begin{algorithm}
    \SetKwInput{Input}{Input}
    \SetKwInput{Output}{Output}
    \Input{Dataset $\D = {(x_i, y_i)}_{i = 1}^T$, regression algorithm $\gA$, target confidence level $1-\alpha$, aggregation function $\phi$, number of bootstrap models $B$, prediction horizon $h$, test data $\D = {(x_t, y_t)}_{t = T+1}^{T+ T_1}$}
    \Output{Ensemble prediction intervals $\{\Gamma^{\phi, 1-\alpha}_{T,t}(x_t)\}_{t = T+1}^{T+ T_1}$.}
    \hrulefill\\
    \tcp{Train bootstrap predictors}
    \For{$b=1,\ldots,B$}
    {
    Sample with replacement an index set $S_b = (i_1,\ldots, i_T)$ from indices $(1,\ldots, T)$;\\
    Compute $\hat{f}^b = \gA( \{(x_i, y_i)|i\in S_b \} )$
    }   
    
    \tcp{Calculate nonconformity scores}
    Initialize $\mathbf{s} = \{\}$; \\
    \For{$i=1,\ldots, T$}
    {
    $\hat{f}^\phi_{-i} (x_i) = \phi ( \{\hat{f}^\phi(x_i) | x \notin S_b \} ) $; \\
    Compute $\hat{s}_i^\phi = |y_i - \hat{f}^\phi_{-i}(x_i) |$;\\
    $\mathbf{s} = \mathbf{s} \cup \{\hat{s}_i^\phi\} $;\\
    }
    
    \tcp{Prediction}
    \For{$t = T+1, \ldots, T+ T_1$}
    {
    Let $\hat{f}^\phi_{-i} (x_t) = (1-\alpha) \text{ quantile of } \{ \hat{f}^\phi_{-i} (x_i)\}_{i=1}^T$; \\
    Let $w_t^\phi = (1-\alpha) \text{ quantile of } \mathbf{s}$; \\
    \KwRet{$\Gamma^{\phi, 1-\alpha}_{T,t}(x_t) = [\hat{f}^\phi_{-i} \pm  w_t^\phi] $};\\

    \tcp{Re-calibration using recent nonconformity scores}

    \If{$t-T = 0 \mod h$}
    { \For{$j=t-h, \ldots, t-1$}
        {
        Compute $\hat{s}_j^\phi = |y_j - \hat{f}^\phi_{-j}(x_t) |$;\\
        $\mathbf{s} = (\mathbf{s} - \{\hat{s}_1^\phi\}) \cup \{\hat{s}_j^\phi\} $ and reset index of $\mathbf{s}$
        }
    }
    }
 \caption{Sequential Distribution-free Ensemble Batch
Prediction Intervals (EnbPI)}
 \label{alg:enbpi}
\end{algorithm}


\subsubsection{Adaptive Conformal Inference Under Distribution Shift (ACI \cite{gibbs2021adaptive})}  
The EnbPI algorithm adapts to the inherent distribution change in time series through a fixed-window re-calibration process. One naturally wonders if there are more principled ways to perform online re-calibration.
Gibbs et al.'s approach to a shifting underlying data distribution, Adaptive conformal inference, is to re-estimate the nonconformity score function $S(\cdot)$ quantile function $\hat{Q}(\cdot)$ to align with the most recent observations: one can assume at each time $t$ we are given a fitted score function $S_t(\cdot)$ and corresponding quantile function $\hat{Q}_t(\cdot)$. We define the miscoverage rate of the prediction set $\Gamma_t^{1-\alpha}(X_t):= \{y: S_t(X_t,y) \leq \hat{Q}_T(1-\alpha) \}$ to be 
\[
M_t(\alpha) := \Prob(S_t(X_t, Y_t) > \hat{Q}_t (1-\alpha))
\]
Note that $M_t(\alpha) < \alpha$ is equivalent to the validity condition in Equation \ref{eq:validity}. In other words, we want to keep the miscoverage rate of a $\alpha$ prediction region less than $\alpha$. We assume the quantile function $ \hat{Q}_t (1-\alpha)$ is continuous, non-decreasing such that $\hat{Q}_t(0) = -\infty$ and $\hat{Q}_t(1) = \infty$. As a result, $M_t(\cdot)$ will be non-decreasing on $[0,1]$ with $M_t(0) = 0$ and $M_t(1) = 1$. Then there exists a $\alpha^*_t$ defined as
\[
\alpha^*_t := \sup \{\beta \in [0,1]: M_t(\beta) \leq \alpha \}
\]
This means if we correctly calibrate the argument to $\hat{Q}_t (\cdot)$ we can achieve the coverage guarantee. \cite{gibbs2021adaptive} performs the calibration with an online update - we look at the miscoverage rate on our calibration set and increase $\alpha^*_t$ if the algorithm currently over-covers, and decrease it if $Y_t$ is under-covered. 
\begin{equation}
    \alpha_{t+1} := \alpha_t + \gamma(\alpha-\text{err}_t)
\label{eq:aci}
\end{equation}

where 
\[
\text{err}_t :=  \begin{cases}
			1, & \text{if } Y_t \notin \hat{\Gamma}_t(\alpha_t), \; \hat{\Gamma}_t(\alpha_t) = \{y:\mathbf{s}_t(X,y) \leq \hat{Q}_t(1-\alpha) \}  \\
            0, & \text{otherwise}
		 \end{cases}
\]
and the step size $\gamma$ is a hyper-parameter that is a trade-off between adaptability and stability. We essentially re-calibrate our conformal predictor by learning the $\alpha^*_t$ online. Because Equation \ref{eq:aci} is a contraction, we have the following coverage result on ACI.

\begin{lemma}[Validity of ACI (Proposition 4.1 in \cite{gibbs2021adaptive})]
With probability one we have that for all $T \in \sN$
\[
\left| \frac{1}{T}\sum^{T}_{t=1} \text{err}_t - \alpha \right| \leq \frac{\max\{\alpha_1, 1-\alpha_1\} + \gamma}{\gamma T}
\]
In particular, asymptotically we have
\[
\lim_{T\rightarrow\infty}\frac{1}{T}\sum_{t=1}^T\text{err}_t = \alpha
\]
\end{lemma}

\begin{figure}
    \centering
    \includegraphics[width=0.8\linewidth]{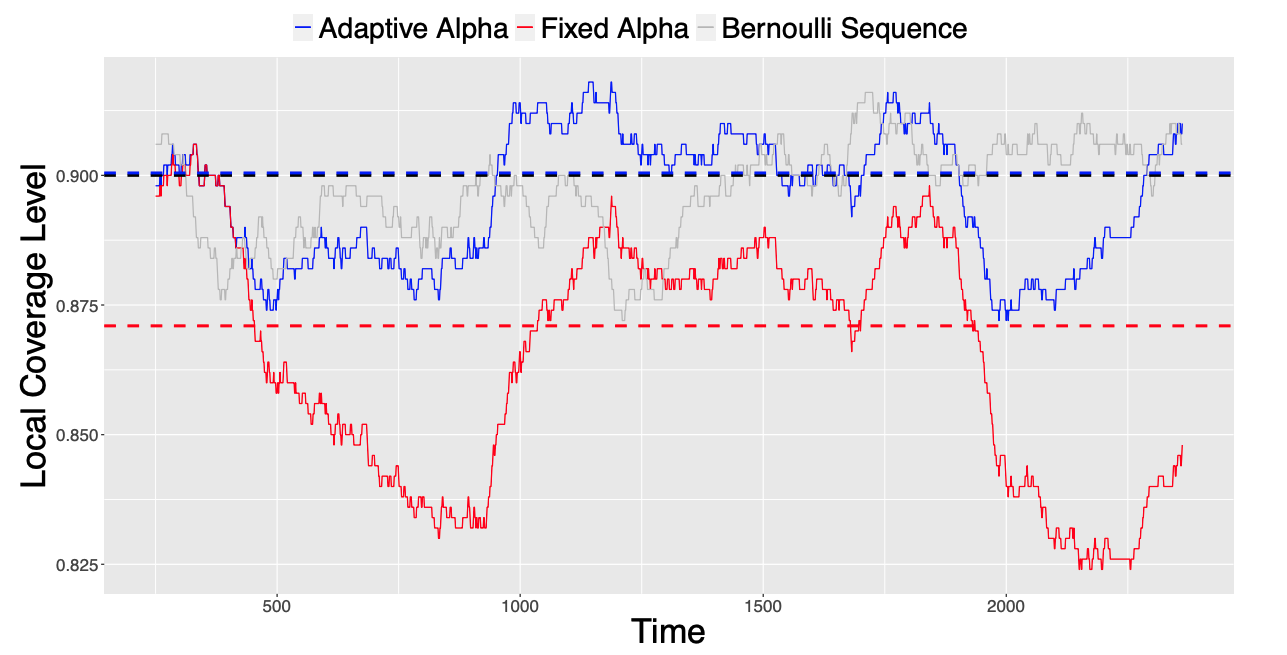}
    \caption{An illustration of the ACI algorithm in action on county level election data, figure from \protect\cite{gibbs2021adaptive}. Local coverage frequencies of adaptive conformal (blue), vanilla conformal (red), and an i.i.d. Bernoulli(0.1) sequence (grey) for reference of stochasticity. Colored dotted lines show the average coverage across all time points, while the black line indicates the target coverage level of $1-\alpha = 0.9$. }
    \label{fig:aci}
\end{figure}



\subsection{Conformal algorithms for multiple exchangeable time series}
The algorithms covered above are designed to learn from a single time series where distribution shifts are present. In this section, we study the setting where datasets contain many time series whose shared patterns can be potentially exploited for improved prediction. This setting is commonly seen in medical data, traffic, and recommender system settings where research in probabilistic forecasting has attracted much interest \cite{salinas2020deepar}. Figure \ref{fig:paradigm} illustrates how the two data paradigms differ. 

\begin{figure}[ht!]
    \centering
    \includegraphics[width=0.8\linewidth]{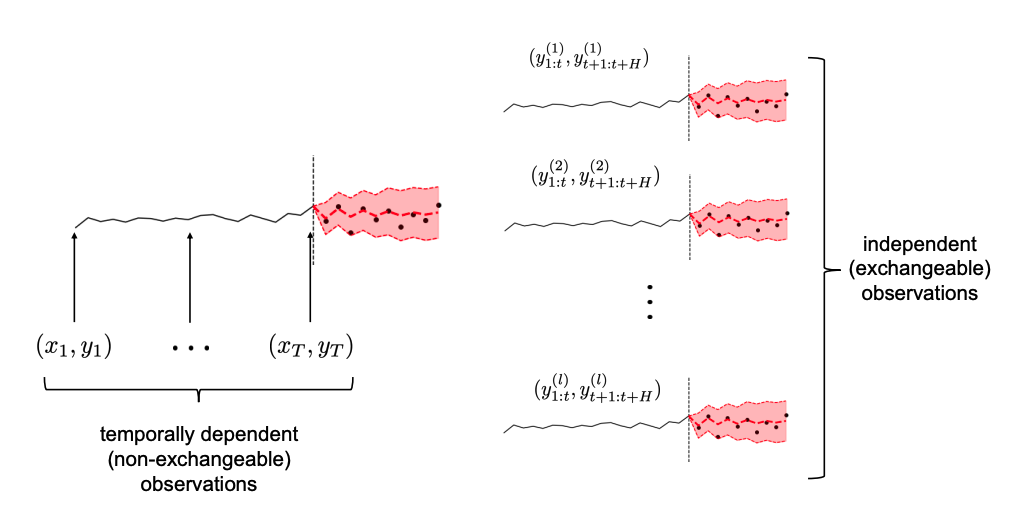}
    \caption{Figure from \protect\cite{Schaar2021conformaltime} explaining two different paradigms of time series data. The left is when data is comprised of a single time series, whose observations are individual time-steps within the time series. The right side is an example of a data set of multiple time series where each time series are treated as an observation. Independence of time series means they are exchangeable.}
    \label{fig:paradigm}
\end{figure}

The formal setup for multi-horizon time series is as follows. We denote the time series dataset as $\mathcal{D} = \{ (\mathbf{x}_{1:t}^{(i)}, \mathbf{y}_{t+1:t+k}^{(i)} \}_{i=1}^n$, where $\mathbf{x}_{1:t} \in \mathbf{X}^t$ is $t$ timesteps of input, $\mathbf{y}_{t+1:t+k}  \in \mathbf{Y}^k$ is $k$ timesteps of output. In the traditional time series forecasting setting we have $\mathbf{X} = \mathbf{Y}$, but they do not have to be equal. Given dataset $\mathcal{D}$, a test sample $\mathbf{x}_{1:t}^{(n+1)}$, and a confidence level $1- \alpha$, we want a set of $k$ prediction intervals, $[\Gamma^{1-\alpha}_1, \ldots, \Gamma^{1-\alpha}_k]$, one for each timestep, such that:

\begin{equation}
    \Prob[\: \forall h \in \{1,\ldots, k\}, \: \mathbf{y}_{t+h} \in \Gamma^{1-\alpha}_h\: ] \geq 1-\alpha
\label{eq:ts-validity}
\end{equation}

\subsubsection{Conformal Time-series forecasting} (CF-RNN \cite{Schaar2021conformaltime}) 
The straightforward way to apply conformal prediction to this form of time series is to treat each $\rvy_h$ time step as a separate prediction task. \cite{Schaar2021conformaltime} did exactly so. To ensure that the joint probability over $k$ prediction time steps satisfies the coverage guarantee in equation \ref{eq:ts-validity}, they apply Bonferroni correction \cite{bonferroni1936teoria} when calibrating for each time step. That is, given the confidence level $1 - \alpha$, instead of setting the critical nonconformity score as the $(1-\alpha)$-th quantile of the calibration nonconformity scores as in algorithm \ref{alg:cp}, we chose the  $(1-\alpha/k)$-th quantile. 

Formally, let subscript $h \in [1,\ldots, k]$ be the timestep index of variables - $\rvs_h$ denote the set of calibration-set nonconformity score at prediction horizon $h$, and $\hat{f}(\rvx)_h$ denote the prediction by learned model $\hat{f}$ given input $\rvx$ at prediction horizon $h$. \cite{Schaar2021conformaltime} constructs the prediction sets as:
\begin{align}
\hat{s}_h &:=  Q(1-\alpha/k, \mathbf{s}_h), \quad h\in [1,\ldots, k]\\
\Gamma^{1-\alpha}_h(\rvx) &:= \left[ \hat{f}(\rvx)_h-\hat{s}_h, \hat{f}(\rvx)_h+\hat{s}_h  \right], \quad h\in [1,\ldots, k]
\label{eq:cfrnn}
\end{align}

\begin{lemma}[Validity of CF-RNN]
The prediction intervals $[\Gamma^{1-\alpha}_1, \ldots, \Gamma^{1-\alpha}_k]$ constructed with Equation \ref{eq:cfrnn} satisfy the coverage guarantee of Equation \ref{eq:ts-validity}.
\end{lemma}

\begin{proof}
The proof follows directly from Boole's inequality. Let $p_i = \Prob[\: \mathbf{y}_{t+i} \notin \Gamma^{1-\alpha}_i \:] $ denote the probability of 
\begin{align*}
    \Prob[\: \exists h \in \{1,\ldots, k\}, \: \mathbf{y}_{t+h} \notin \Gamma^{1-\alpha}_h \: ] 
    = \Prob [ \:\bigcup_{i=1}^k (p_i\leq \frac{\alpha}{k}) \: ] 
    \leq \sum_{i=1}^k (\: \Prob[p_i \leq \frac{\alpha}{k} ] \:)
    = k \cdot \frac{\alpha}{k} = \alpha
\end{align*}
Take the contrapositive and we have the coverage result in Equation \ref{eq:ts-validity}.
\end{proof}

Bonferroni correction does not require any assumptions about dependence among the probabilities of individual timesteps. As a result, the predictive intervals are very conservative. If we require an interval of $95\%$ confidence over 5 timesteps, we would end up using the $99\%$ quantile of nonconformity scores for each timesteps; if it is over 20 timesteps, then the quantile for each timestep goes up to $99.75\%$, resulting in large and often uninformative predictive intervals.  

\begin{algorithm}[t!]
    \SetKwInput{Input}{Input}
    \SetKwInput{Output}{Output}
    \Input{Trained Prediction model $\hat{f}$ producing $k$-step forecasts, calibration set $\mathcal{D}_{cal} = \{ (\mathbf{x}_{1:t}^{(i)}, \mathbf{y}_{t+1:t+k}^{(i)}\}_{i=m+1}^n\}$, new input $\mathbf{x}_{1:t}^{n+1}$, target miscoverage level $\alpha$}
    \Output{Prediction regions $\Gamma^\epsilon_1, \ldots,\Gamma^\epsilon_k$.}
    \hrulefill\\
    \tcp{Calibration Step}
    Initialize conformity scores $s_1 = \{\}, \ldots, s_k = \{\} $\\
    \For{$(\mathbf{x}_{1:t}^i, \mathbf{y}_{t+1:t+k}^i) \in \mathcal{D}_{cal} $}{
        $\hat{\mathbf{y}}_{t+1:t+k}^{i} \leftarrow \hat{f}(\mathbf{x}_{1:t}^{i})$ \\
        \For{ $h=1$ \KwTo $k$}{
        $s_h \leftarrow s_h \bigcup |\hat{\mathbf{y}}_{t+h}^{i} - \mathbf{y}_{t+h}^{i}|$
        }
    }
    \For{ $h=1$ \KwTo $k$}{
     $\hat{s}_h \leftarrow  Q(1-\alpha/k, \mathbf{s}_h) $ \\
    } 
    \tcp{Prediction Step}
    $\hat{\mathbf{y}}^{n+1}_{t+1:t+k} \leftarrow \hat{f}(\mathbf{x}_{1:t}^{n+1})$ \\
    \For{ $h=1$ \KwTo $k$}{
     $\Gamma_h \leftarrow \{ y : \;  |y - \hat{y}^{n+1}_h | < \hat{s}_h \} $ \\
    } 
     \KwRet{$\Gamma^\epsilon_1, \ldots,\Gamma^\epsilon_k$}
 \caption{Conformal Forecasting RNN (CF-RNN) \cite{Schaar2021conformaltime}}
 \label{alg:cfrnn}
\end{algorithm}

\subsubsection{Copula conformal prediction for multivariate time series (our work)}

We note that there are often dependencies between timesteps in a time series, and by modeling these dependencies, we can improve  efficiency of the prediction interval. 

\paragraph{Introduction to Copulas}
We introduce the concept of copulas to leverage such dependencies for sharper prediction regions. Copula is a concept from statistics that describes the dependency structure in a multivariate distribution. We can use copulas to capture the joint distribution for multiple future timesteps. It is widely used for economic and financial forecasting \cite{nelsen2007introduction, patton2012review, remillard2012copula} and is explored in the ML multivariate time series literature \cite{salinas2019high, drouin2022tactis}. 

\begin{definition}[Copula]
Given a random vector $(X_1,\cdots X_k)$, define the marginal cumulative density function (CDF) as $F_i(x) = P(X_i \leq x)$,  the copula of $(X_1, \cdots X_k)$ is the joined CDF of $(U_1,\cdots, U_k) = (F(X_1), \cdots, F(X_k))$, meaning 
\[ C(u_1,\cdots, u_k)  = P(U_1\leq u_1, \cdots, U_k\leq u_k)\]
Alternatively
\[ C(u_1,\cdots, u_k)  = P(X_1\leq F_1^{-1}(u_1), \cdots, X_k\leq F_t^{-1}(u_k)) \]
\end{definition}

In other words, the copula function captures the dependency structure between the variable $X$s; we can view a $k$ dimensional copula $C: [0,1]^k \rightarrow [0,1]$ is a CDF with uniform marginals. A fundamental result in the theory of copula is Sklar's theorem. 

\begin{theorem}[Sklar's theorem]
Given a joined CDF  as $F(X_1,\cdots, X_k)$ and the marginals $F_i(x)$,  there exists a copula such that 
\[F(x_1,\cdots, x_k) = C(F_1(x_1), \cdots, F_k(x_k))\]
for all $x_i\in [-\infty, \infty]$ and $i =1, \cdots, k$.
\label{the:sklar}
\end{theorem}

Sklar's theorem states that for all multivariate distribution functions, we can find a copula function such that the distribution can be expressed using the copula and a mixture of univariate marginal distributions.  Copulas can be parametric or non-parametric; the Gaussian copula, for example, is used for pricing in financial markets \cite{Haugh2016copula}. Figure \ref{fig:copula} illustrates an example of a bivariate copula. To give an example, when all the $X_k$s are independent, the copula function is know as the product copula: $C(u_1, \cdots, u_k) = \Pi_{i=1}^k u_i$. We can bound copula functions as follows.

\begin{theorem}[The Frechet-Hoeffding Bounds]
Consider a copula $C(u_1, \ldots, u_k)$. Then,
\[
\max{\{1-k+\sum_{i=1}^k u_i, 0\} } \leq C(u_1, \ldots, u_k) \leq \min{ \{u_1, \ldots, u_k \}}
\]
\label{the:fhbounds}
\end{theorem}

The Frechet-Hoeffding upper- and lower- bounds are both copulas. In fact, it is notable that Bonferroni correction used in CF-RNN \ref{alg:cfrnn} is precisely the lower bound in theorem \ref{the:fhbounds}.

\begin{figure}[ht!]
    \centering
    \includegraphics[width=\linewidth]{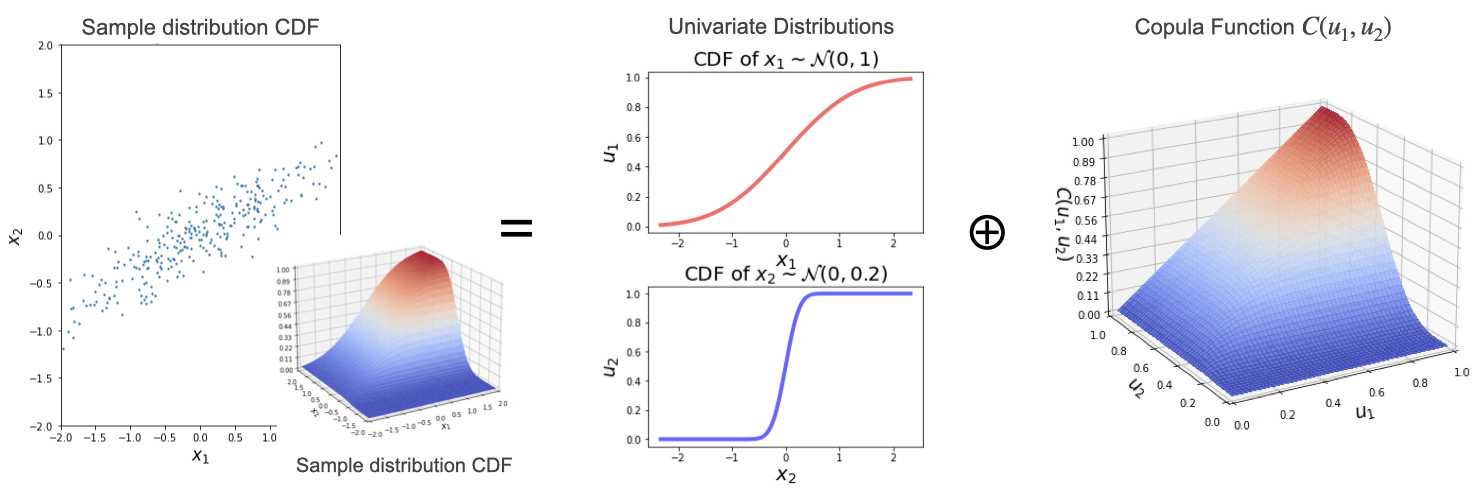}
    \caption{An example copula, where we express a multivariate Gaussian with correlation $\rho = 0.8$ with two univariate distributions and a Copula function $C(u_1, u_2)$.} 
    \label{fig:copula}
\end{figure}

\paragraph{Algorithm of Copula Conformal Prediction for Time Series}
We now introduce how we can leverage copulas in time series prediction. Firstly, recall that from results in section 3 we know given a confidence level $\alpha$ and a prediction model $\hat{f}$, the conformal procedure returns a nonconformity score $s_{1-\alpha}$ such that for a $z=(x,y) \sim \gP$
\[
	P(\gS(z, \hat{f}) \leq s_{1-\alpha} ) \geq 1-\alpha
\]

We can view the procedure as an empirical cumulative distribution function (CDF) for the random variable $\gS(z, \hat{f})$

\begin{equation}
    \hat{F}(s) := P(\gS(z, \hat{f})\leq s) = \frac{1}{|D_{cal}|} \sum_{z^i \in D_{cal} } \mathbbm{1}_{\gS(z^i, \hat{f}) \leq s} 
    \label{eq:cdf-vanila}
\end{equation}

For the multi-step conformal prediction algorithm, we will estimate this empirical CDF for each step of the time series, denoted as 

\begin{equation}
   \hat{F}_t(s_t) := P(\gS(z_t, \hat{f}) \leq s_t) \;  \text{for } t \in 1, \ldots, k
    \label{eq:cdf}
\end{equation}

For each prediction timestep $ 1,\ldots, k$, we want to find a significance level $\epsilon_t$ such that the entire predicted time series trajectory is covered by the intervals with confidence $1-\epsilon$. Following Sklar's Theorem (Theorem \ref{the:sklar}):

\begin{align*}
    F(s_1, \ldots, s_k) &= C(F_1(s_1), \ldots, F_k(s_k))\\
    &= C(1-\epsilon_1, \ldots, 1-\epsilon_k)\\
    &= 1 - \epsilon
\end{align*}

In this work, we adopt the empirical copula \cite{empiricalcop1} as our default copula. The empirical copula is a nonparametric method of estimating marginals directly from observation, and hence does not introduce any bias. For the joint distribution of a time series with $k$ timesteps, the copula of a vector of probabilities $\mathbf{u} \in [0,1]^k$ is defined as
\begin{equation}
C_{\text{empirical}} (\mathbf{u}) =  \frac{1}{|\Dcal|}\sum_{i\in \Dcal} \mathbbm{1}_{\mathbf{u}^i <\mathbf{u}} = \frac{1}{|\Dcal|}\sum_{i\in \Dcal}^{n} \prod_{t=1}^{k}  \mathbbm{1}_{\mathbf{u}^i_t <\mathbf{u}_t}
\label{eq:empiricalcopula}
\end{equation}

Here, the $\mathbf{u}_i$s are the cumulative probabilities for each data in the calibration set $\mathcal{D}_{cal}$ of size $n-m$. 
\[
     \mathbf{u}^i = (u^i_1, \ldots, u^i_t) = (\hat{F}_1(s^i_1), \ldots, \hat{F}_1(s^i_t)), \; i\in \{m+1, \ldots, n\} 
\]

\cite{messoudi2021copula} demonstrated that the empirical copula performs well for multi-target regression, flexible to various distributions. One may pick other parametric copula functions, such as the Gaussian copula, to introduce inductive bias and improve sample efficiency when calibration data is scarce. 

Now, to fulfill the validity condition of equation \ref{eq:validity}, we only need to find  
\[
 \mathbf{u^*}=(\hat{F}_1(s^*_1), \ldots, \hat{F}_t(s^*_k)) \; \text{such that } \; C_{\text{empirical}} (\mathbf{u^*}) \geq 1-\epsilon
\]

We can find the $s^*$s through any search algorithm. The prediction region for each timestep is constructed as the set of all $y_t \in \textbf{y}_t$ such that the nonconformity score is less than $s_t^*$. Algorithm \ref{alg:cpts} summarizes the CoupulaCPTS algorithm. The difference between CF-RNN (algorithm \ref{alg:cfrnn}) and CopulaCPTS (algorithm \ref{alg:cpts}) is line 10-12 of the latter - instead of calculating the quantile of nonconformity scores for each timestep individually, CopulaCPTS searches for the nonconformity scores that jointly ensures validity of the copula. 

\begin{algorithm}[t!]
    \SetKwInput{Input}{Input}
    \SetKwInput{Output}{Output}
    \Input{Trained Prediction model $\hat{f}$ producing t-step forecasts, calibration set $\mathcal{D}_{cal} = \{ (\mathbf{x}_{1:t}^{(i)}, \mathbf{y}_{t+1:t+k}^{(i)}\}_{i=m+1}^n\}$, new input $\mathbf{x}_{1:t}^{n+1}$, target significant level $\epsilon$}
    \Output{Prediction regions $\Gamma^\epsilon_1, \ldots,\Gamma^\epsilon_k$.}
    \hrulefill\\
    \tcp{Calibration Step}
    Initialize conformity scores $s_1 = \{\}, \ldots, s_k = \{\} $\\
    \For{$(\mathbf{x}_{1:t}^i, \mathbf{y}_{t+1:t+k}^i) \in \mathcal{D}_{cal} $}{
        $\hat{\mathbf{y}}_{t+1:t+k}^{i} \leftarrow \hat{f}(\mathbf{x}_{1:t}^{i})$ \\
        \For{ $h=1$ \KwTo $k$}{
        $s_h \leftarrow s_h \bigcup |\hat{\mathbf{y}}_{t+h}^{i} - \mathbf{y}_{t+h}^{i}|$
        }
    }
    Construct $\hat{F}_1 \ldots \hat{F}_k$ as equation \ref{eq:cdf}. \\
    Construct copula $C(\mathbf{u})$ as equation \ref{eq:empiricalcopula}.\\
     Search for $s^*_1, \ldots, s^*_k$  such that $C(\hat{F}_1(s^*_1), \ldots, \hat{F}_k(s^*_k)) \geq 1-\epsilon$\\
    \tcp{Prediction Step}
    $\hat{\mathbf{y}}^{n+1}_{t+1:t+k} \leftarrow \hat{f}(\mathbf{x}_{1:t}^{n+1})$ \\
    \For{ $h=1$ \KwTo $k$}{
     $\Gamma_h \leftarrow \{ y : \;  |y - \hat{y}^{n+1}_h | < s^*_h \} $ \\
    } 
     \KwRet{$\Gamma^\epsilon_1, \ldots,\Gamma^\epsilon_k$}
 \caption{Copula Conformal Prediction for Time Series (Copula CPTS)}
 \label{alg:cpts}
\end{algorithm}

\paragraph{Validity of Copula CPTS}

\begin{lemma}[Validity of Copula CPTS]
The prediction regions provided by algorithm \ref{alg:cpts} $[\Gamma^{1-\alpha}_1, \ldots, \Gamma^{1-\alpha}_k]$ satisfies the coverage guarantee $\Prob[\: \forall h \in \{1,\ldots, k\}, \: \mathbf{y}_{t+h} \in \Gamma^{\epsilon}_h\: ] \geq 1-\epsilon$.
\end{lemma}
\begin{proof}
We estimated $\mathbf{u^*} = (\hat{F}_1(s^*_1), \ldots, \hat{F}_t(s^*_k))$ such that
\begin{align}
    C_{\text{empirical}} (\mathbf{u^*}) = \displaystyle \mathop{\mathbb{E}}_u  \prod_{t=1}^{k}  \mathbbm{1}_{\mathbf{u}^*_t <\mathbf{u}_t} \geq 1-\epsilon
\end{align}
Let $(\mathbf{x}_{1:t}^{n+1}, \mathbf{y}_{t+1:t+h}^{n+1}) \sim \mathbf{X}\times \mathbf{Y}$ be a new data point. Denote $\hat{\mathbf{y}} = \hat{f}(\mathbf{x}_{1:t}^{n+1})$, the prediction given by the trained model. Because the $\hat{F}_h$ functions and $C_{\text{empirical}}$ are monotonously increasing, we have:
\begin{align*}
    P(\: \forall h \in \{1,\ldots, k\}, \: \mathbf{y}_{t+h} \in \Gamma^{\epsilon}_h\: ) 
    &= C(\hat{F}_1(s^{n+1}_1), \ldots, \hat{F}_k(s^{n+1}_k))\\
    &\geq  C( \hat{F}_1(s^*_1), \ldots, \hat{F}_k(s^*_k)) \\
    &\geq 1-\epsilon
\end{align*} 
\end{proof}

\begin{figure}[ht]
\centering
\begin{subfigure}[b]{0.33\linewidth}
  \includegraphics[width=\linewidth]{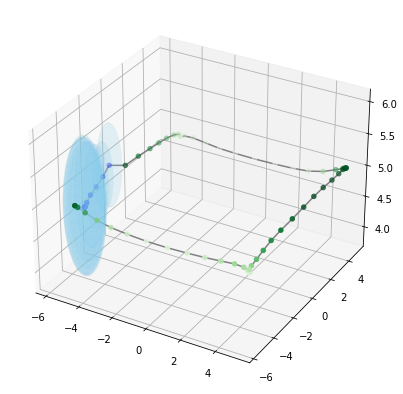}
  \caption{Copula-EncDec}
  \label{fig:cal001}
\end{subfigure}%
\begin{subfigure}[b]{0.33\linewidth}
  \includegraphics[width=\linewidth]{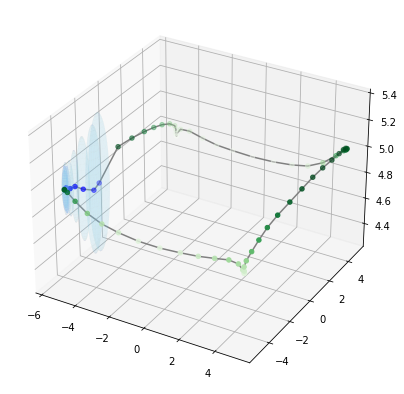}
  \caption{MC Dropout}
  \label{fig:cal002}
\end{subfigure}%
\begin{subfigure}[b]{0.33\linewidth}
  \includegraphics[width=\linewidth]{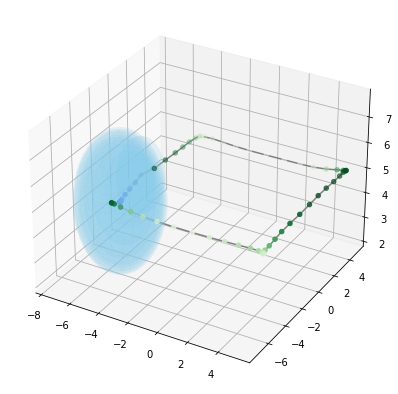}
  \caption{CF-RNN}
  \label{fig:cal003}
\end{subfigure}
\caption{Visualization of the predictions of the Copula method. The underlying models are trained on a drone simulation dataset with dynamics noise; the prediction regions have confidence level $1- \alpha = 99\%$. With copula our method (a) produces a more consistent, expanding cone of uncertainty compared to MC-dropout (b) sharper one compared to CF-RNN (c). }
\label{fig:drone}
\end{figure}


\section{Decision making under (non-parametric) uncertainty}

\subsection{Sample-efficient Safety Assurances with CP}


Many important problems involve decision making under uncertainty, including vehicle collision avoidance, wildfire management, and disaster response. When designing automated systems for making or recommending decisions, it is important to account for uncertainty inherent in the system or learned machine learning predictions. There is rich literature on incorporating UQ in a decision making framework, exemplary works include utility theory \cite{fishburn1968utility}, risk-based decision making \cite{lounis2016risk}, robust optimization \cite{ben2009robust}, and reinforcement learning \cite{chua2018deep}. We will explore methods that leverage conformal prediction to achieve provably safe decision making. 

The properties of conformal prediction, namely (1) agnostic to prediction model (2) agnostic to underlying data distribution and (3) finite sample guarantees, makes it a particularly useful technique for safety assurance. Existing work on probabilistic safety guarantees most commonly adopt a Bayesian framework \cite{dhiman2021control, fisac2018probabilistically}, but many recent work has explored conformal methods for better calibrated UQ of learned dynamics. \cite{sun2020learning}, for example, uses conformal prediction to give probabilistic guarantees on the convergence of the tracking error. \cite{farid2022failure} showed that conformal methods achieve much better sample efficiency than PAC-Bayes guarantees in the case of predicting failures in vision-based robot control. \cite{johnstone2021conformal} explored use of conformal prediction sets for robust optimization and found that they are both more calibrated and more efficient than uncertainty sets based on normality. \cite{eliades2019applying} leverage the confidence provided by CP to make controls to an exoskeleton more robust.  

\begin{definition}[$\epsilon$-safety]
Let $Z$ denote the state an agent and $\phi$ a safety score with threshold $\phi_0$. For some $0<\epsilon<1$, we say the agent is $\epsilon$-safe with respect to $Z,\phi$ and $\phi_0$ if
\[
\Prob[\phi(Z) \leq \phi_0] \geq 1-\epsilon
\]
\end{definition} 

One can make guarantee about the false positive rate for a warning system: \cite{luo2021sample} introduced the paradigm of achieving $\epsilon$-safety of warning system with regard to a surrogate safety score $\phi$ (\ref{def:epsafety}). The authors deployed conformal prediction in experiments of a driver alert safety system and a robotic grasping system, showing that the conformal guarantees hold in practice, issuing very little ($<1\%$) false positive alerts

\begin{definition}[$\epsilon$-safety of warning systems]
Let $Z$ denote the state an agent and $\phi$ a safety score with threshold $f_0$. For some $0<\epsilon<1$, we say a warning system $w$ is $\epsilon$-safe with respect to $Z,\phi$ and $\phi_0$ if
\[
\Prob[w(Z) =1 \;| \; \phi(Z) \leq \phi_0] \geq 1-\epsilon
\]
\label{def:epsafety}
\end{definition}

We briefly note on the difference between safety guarantees provided in the framework of PAC learning and by conformal prediction (see \cite{luo2021sample} for algorithmic and mathematical). PAC learning assumes i.i.d. distribution of data, and requires $\Theta(1/\epsilon^2)$ samples to achieve an \textit{i.i.d.} $\epsilon$-safety guarantee with probability $1-\delta$; conformal prediction, on the other hand, only assumes that data are exchangeable, and requires $\Theta(1/\epsilon)$ samples to achieve a \textit{marginal} $\epsilon$-safety guarantee that always holds. We will elaborated on these difference between the guarantees in Section 4.2. To summarize, their use cases differ in that: conformal learning requires much weaker assumptions and fewer samples, whereas PAC learning offers stronger guarantees when its assumptions and sample complexity requirements are met.


\subsection{The Limits of Distribution-free UQ: Marginal vs Conditional Coverage}

An important distinction to be made between marginal coverage guarantees and conditional coverage \cite{vovk2012conditional} guarantees. So far we have been focused on the \textit{marginal coverage} guarantee formalized in Equation \ref{eq:validity}:
\begin{equation}
    \Prob_{(X,Y) \sim \gP} (Y \in \Gamma^{1-\alpha}(X)) \geq 1-\alpha
\end{equation}
This means that the probability that $ \Gamma^{1-\alpha}(X)$ covers the true test value $Y$ is at least $1-\alpha$, on average, over a random draw of the training and test data from the distribution $(X,Y) \sim \gP$. In other words, the marginal coverage guarantee is a guarantee on the \textit{average} case. This is to be distinguished with \textit{conditional coverage}, where
\begin{equation}
    \Prob_{(X,Y) \sim \gP} (Y \in \Gamma^{1-\alpha}(X)|X=x) \geq 1-\alpha
\end{equation}
This on the other hand means that the probability of $\Gamma^{1-\alpha}$ covering at a \textit{fixed} test point $X=x$ is at least $1-\alpha$. An illustration of the difference can be found in figure \ref{fig:cond}. In practical settings, marginal coverage is often not sufficient, as it leaves open the possibility that entire regions of test points are receiving inaccurate predictions - a robot would reasonably hope that the predictions they receive is accurate for its specific circumstances, and would not be comforted by knowing that the inaccurate information they might be receiving will be balanced out by some other robots’ highly precise prediction \cite{romano2019malice}. Conditionally valid predictions are also tied to some definitions of algorithmic fairness \cite{bastani2022practical}.

\begin{figure}
    \centering
    \includegraphics[width=0.8\linewidth]{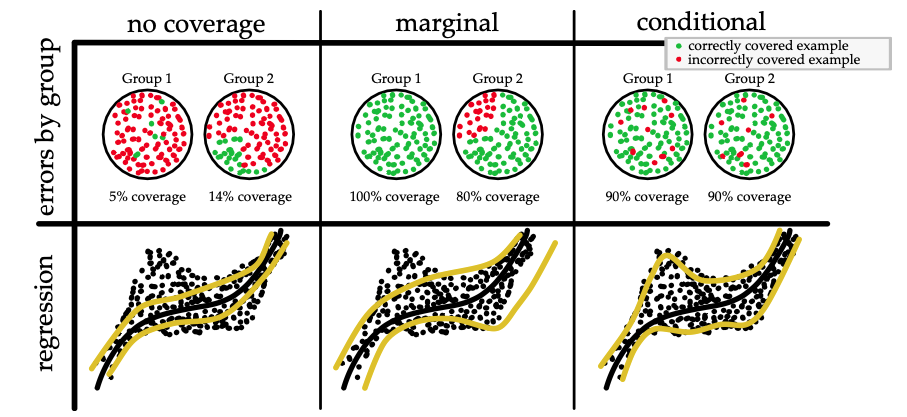}
    \caption{An illustration of the difference between marginal and conditional coverage from \protect\cite{angelopoulos2021gentle}.}
    \label{fig:cond}
\end{figure}

Conditional coverage is a stronger property than the marginal coverage that conformal prediction is guaranteed to achieve. It is proven in \cite{vovk2012conditional,lei2014distribution} that, if we do not place any assumptions on $\gP$, exact conditional coverage guarantee is impossible to achieve. We use $\gP_X$ to denote the marginal distribution of $\gP$ on $X$, and the function $Leb(\cdot)$ to denote the Lebesgue measure in $\R^d$ where $d\geq 1$. Theorem \ref{the:impossible} presents that for almost all nonatomic points $x \sim \gP_X$, a conditionally valid prediction interval has infinite expected length. Intuitively, this means the coverage guarantees of conformal prediction are only applicable for the \textit{average case} and not for the \textit{worst case}, and that there may exist subspaces of the sample space $\mathbf{Z}$ where the prediction regions have poor coverage.  

\begin{theorem}[Impossibility of finite sample conditional validity \cite{lei2014distribution}]
Let $ \Gamma^{1-\alpha}$ be a function such that 
\[
 \Prob_{(X,Y) \sim \gP} (Y \in \Gamma^{1-\alpha}(X)|X=x) \geq 1-\alpha
\]
Then for all distributions $\gP$ it holds that 
\[
\E \left[ Leb(\Gamma^{1-\alpha}(x)) \right] = \infty
\]
at almost all points x asside from the atoms of $\gP_X$.
\label{the:impossible}
\end{theorem}

This is known as the "limits of distribution-free conditional predictive inference" \cite{foygel2021limits}. The negative result has motivated the community to consider approximate versions of the conditional coverage property. \cite{romano2019malice} provides group conditional guarantees for disjoint groups by calibrating separately on each group. \cite{foygel2021limits} provides guarantees that are valid conditional on
membership in intersecting subgroups $\gG$. \cite{bastani2022practical} proposes an algorithm that provides a convergence rate lower bound on conditional coverage on for arbitrary subsets of the input feature space — possibly intersecting, conditional on membership in each of these subsets. They relax the exchangability assumption by considering the adversarial setting, where the order of $\mathcal{D}$ is chosen by an adversary, and there by proving guarantees for the \textit{worst} case. We refer readers to these works for more detailed exposition of their algorithm and theory.

\section{Conclusion and Future Work}

This survey introduces conformal algorithms for uncertainty quantification in the spatiotemporal setting.
Section 2 provides an overview of existing uncertainty quantification methods; section 3 presents the theory of conformal prediction, and studied in detail a few conformal prediction algorithms developed for the time series setting, looking respectively in the case of (1) data generated from one time series with distribution shift and (2) multiple exchangeable time series. In section 4, we explored how conformal methods are used in decision making and where the limits may lie. 

We highlight the merit of conformal prediction algorithms - they construct predictions sets that have finite-sample validity guarantees that hold for any prediction algorithm and underlying data distribution. Moreover, as we have introduced in this survey, the prediction sets are able to be re-calibrated in an online fashion in the presence of distribution shifts. Many opportunities of future work are present on both the algorithm and application front, as we will elaborate below.

In section 3.3.2 we present our work on improving efficiency of the prediction sets in the independent time-series setting, motivated by works in probabilistic modeling of vehicle trajectories. Future work includes developing an online re-calibration procedure for the copula and proving its effectiveness. This setting can be further extended to producing joint prediction sets for multiple targets or agents where their relationship evolves over time, for example in multi-agent trajectory prediction.

Developing methods to leverage distribution-free UQ technique for automated decision making is also an important future direction.  Conformal prediction has found extensive use in medical research and personal medicine \cite{izbicki2019flexible, Schaar2021conformaltime,teng2021t, lei2012distribution, eklund2015application}, where the scientists and physicians can use the calibrated intervals to understand model outputs and make better decisions. On the other hand, using CP for automated decision making remains a large space for exploration. For example, in a multi-agent planning setting, one can use conformal prediction sets to produce trajectories with similar safety guarantees as the $\epsilon$-safety introduced in section 4. The online adaptive time series methods introduced in this survey will also be meaningful in the safety setting.

\bibliographystyle{abbrv}
\bibliography{ref.bib}

\end{document}